\documentclass{article} % For LaTeX2e
\usepackage{nips15submit_e}
\usepackage{times}
\usepackage{graphicx} % more modern
\usepackage{subfigure}

\usepackage{amsthm}
\usepackage{amsmath}
\usepackage{amsfonts}
\usepackage{color}
\usepackage{natbib}

\newcommand{\T}{^\top}
\newcommand{\R}{\mathbb R}

\newcommand{\ExpBy}[3]{\mathbb{E}^{#2}\left[ \left. #1 \right| #3 \right]}

\newtheorem{Theorem}{Theorem}

\newtheorem{corollary}{Corollary}

\nipsfinalcopy

\title{Emphatic TD Bellman Operator is a Contraction}
\author{Assaf Hallak \\
Electrical Engineering Dept.\\
Technion\\
\small{\texttt{ifogph@gmail.com}}
\And
Aviv Tamar \\
Electrical Engineering Dept.\\
Technion\\
\small{\texttt{avivt@tx.technion.ac.il}}
\And
Shie Mannor \\
Electrical Engineering Dept.\\
Technion\\
\small{\texttt{shie@ee.technion.ac.il}}
}

\begin{document}
%\twocolumn[
%\icmltitle{Emphatic TD Bellman Operator is a Contraction}
%\icmlauthor{Assaf Hallak}{ifogph@gmail.com}
%\icmlauthor{Aviv Tamar}{avivt@tx.technion.ac.il}
%\icmlauthor{Shie Mannor}{shie@ee.technion.ac.il}
%\icmladdress{Electrical Engineering Department,
%The Technion - Israel Institute of Technology, Haifa, Israel 32000}
%\vskip 0.3in
%]
\maketitle

\begin{abstract}
Recently, \citet{SuttonMW15} introduced the emphatic temporal differences (ETD) algorithm for off-policy evaluation in Markov decision processes. In this short note, we show that the projected fixed-point equation that underlies ETD involves a contraction operator, with a $\sqrt{\gamma}$-contraction modulus (where $\gamma$ is the discount factor). This allows us to provide error bounds on the approximation error of ETD. To our knowledge, these are the first error bounds for an off-policy evaluation algorithm under general target and behavior policies.
\end{abstract}

\section{Introduction}
In Reinforcement Learning (RL; \citealt{sutton_reinforcement_1998}), \emph{policy-evaluation} refers to the problem of evaluating the value function -- a mapping from states to their long-term discounted return under a given policy, using sampled observations of the system dynamics and reward. Policy-evaluation is important both for assessing the quality of a policy, but also as a sub-procedure for policy optimization \citep{sutton_reinforcement_1998}.

For systems with large or continuous state-spaces, an exact computation of the value function is often impossible. Instead, an \emph{approximate} value-function is sought using various function-approximation techniques (\citealt{sutton_reinforcement_1998}; a.k.a. approximate dynamic-programming; \citealt{Ber2012DynamicProgramming}). In this approach, the parameters of the value-function approximation are tuned using machine-learning inspired methods, often based on the \emph{temporal-difference} idea (TD;\citealt{sutton_reinforcement_1998}).

 The method generating the sampled data leads to two different types of policy evaluation. In the \emph{on-policy} case, the samples are generated by the \emph{target-policy} -- the policy under evaluation, while in the \emph{off-policy} setting, a different \emph{behavior-policy} generates the data. In the on-policy setting, TD methods are well understood, with classic convergence guarantees and approximation-error bounds, based on a contraction property of the projected Bellman operator underlying TD \citep{BT96}. For the off-policy case, however, standard TD methods no longer maintain this contraction property, the error bounds do not hold, and these methods may even diverge \citep{baird1995residual}.

Recently, \citet{SuttonMW15} proposed the \emph{emphatic TD} (ETD) algorithm: a modification of the TD idea that can be shown to converge off-policy \citep{yu2015etd}. In this paper, we show that the projected Bellman operator underlying ETD also possesses a contraction property, which allows us to derive approximation-error bounds for ETD.

In recent years, several different off-policy policy-evaluation algorithms have been proposed and analyzed, such as importance-sampling based least-squares TD \citep{yu2012least}, gradient-based TD \citep{sutton2009fast}, and ETD \citep{SuttonMW15}. While these algorithms were shown to converge, to our knowledge there are no guarantees on the \emph{error} of the converged solution. The only exception that we are aware of, is a contraction-based argument for importance-sampling based LSTD, under the restrictive assumption that the behavior and target policies are very similar \citep{bertsekas2009projected}. This paper presents the first approximation-error bounds for off-policy policy evaluation under general target and behavior policies.

\section{Preliminaries}
We consider an MDP $M=(S, A, P, R, \gamma, \rho)$, where $S$ is the state space, $A$ is the action space, $P$ is the transition probability matrix, $R$ is the reward function, $\gamma\in [0,1)$ is the discount factor, and $\rho$ is the initial state distribution.

Given a target policy $\pi$, our goal is to evaluate the \emph{value function}:
\begin{equation*}
    V^\pi(s) \doteq \ExpBy{\sum_{t=0}^\infty R(s_t,a_t)}{\pi}{s_0 = s}.
\end{equation*}

Temporal difference methods \citep{sutton_reinforcement_1998}, approximate the value function by
\begin{equation*}
    V^\pi(s) \approx \theta \T \phi(s)  ,
\end{equation*}
where $\phi(s)\in \R^n$ are state features, and $\theta \in \R^n$ are weights, and use sampling to find a suitable $\theta$.
Let $\mu$ denote a behavior policy that generates the samples $s_0,a_0,s_1,a_1,\dots$ according to $a_t \sim \mu(\cdot|s_t)$ and $s_{t+1}\sim P(\cdot|s_t,a_t)$. We denote by $\rho_t$ the ratio $\pi(a_t|s_t)/\mu(a_t|s_t)$, and we assume, similarly to \citet{SuttonMW15}, that $\mu$ and $\pi$ are such that $\rho_t$ is well-defined for all $t$.

Let $T^\pi$ denote the Bellman operator for policy $\pi$, given by $$T^\pi V = R_\pi + \gamma P_\pi V,$$ where $R_\pi$ and $P_\pi$ are the reward vector and transition matrix induced by policy $\pi$, and let $\Phi$ denote a matrix whose columns are the feature vectors for all states. Let $d_\mu$ and $d_\pi$ denote the stationary distributions over states induced by the policies $\mu$ and $\pi$, respectively. For some $d\in \R^{|S|}$ satisfying $d>0$ element-wise, we denote by $\Pi_d$ a projection to the subspace spanned by $\phi(s)$ with respect to the $d$-weighted Euclidean-norm.

Similarly to \cite{SuttonMW15}, we divide the analysis to the `pure bootstrapping' case $\lambda=0$, and the more general case with $\lambda \in [0,1)$. The ETD($0$) algorithm  iteratively updates the weight vector $\theta$ according to:
\begin{equation*}
\begin{split}
    \theta_{t+1} &:= \theta_t + \alpha F_t \rho_t (R_{t+1} + \gamma \theta_t \T \phi_{t+1} - \theta_t \T \phi_t) \phi_t \\
    F_t &= \gamma \rho_{t-1}F_{t-1} + 1, \quad F_0 = 1.
\end{split}
\end{equation*}

The emphatic weight vector $f$ is defined by
\begin{equation}\label{eq:f}
    f \T = d_\mu \T (I - \gamma P_\pi)^{-1}.
\end{equation}

The ETD($\lambda$) algorithm iteratively updates the weight vector $\theta$ according to
\begin{equation*}
\begin{split}
\theta_{t+1} &:= \theta_t + \alpha (R_{t+1} + \gamma \theta_t \T \phi_{t+1} - \theta_t \T \phi_t) e_t \\
e_t &= \rho_t (\gamma \lambda e_{t-1} + M_t \phi_t ), \quad e_{-1} = 0 \\
M_t &= \lambda i(S_t) + (1-\lambda)F_t \\
F_t &= \rho_{t-1} \gamma F_{t-1} + i(S_t), \quad F_0 = i(S_0),
\end{split}
\end{equation*}
where $i:S\rightarrow \mathcal{R}^+$ is a known given function signifying the importance of the state. Note that \citet{SuttonMW15} consider state-dependent discount factor $\gamma(s)$ and bootstrapping parameter $\lambda(s)$, while in this paper we consider the special case where $\gamma$ and $\lambda$ are constant.

The emphatic weight vector $m$ is defined by
\begin{equation}\label{eq:m}
m \T = \textbf{i} \T (I - P^\lambda_\pi)^{-1},
\end{equation}
where:
\begin{equation*}
\begin{split}
\textbf{i}(s) &= i(s) \cdot d_\mu (s), \\
P^\lambda_\pi &= I - (I-\gamma \lambda P_\pi)^{-1} (I-\gamma P_\pi).
\end{split}
\end{equation*}
Notice that in the case of general $\lambda$, the Bellman operator is:
\begin{equation}
T^{(\lambda)} v = (I - \gamma \lambda P_\pi)^{-1}r_\pi + P^\lambda_\pi v.
\end{equation}

\citet{Mahmood2015ETD} show that ETD converges to some $\theta^*$ that is a solution of the projected fixed-point equation:
\begin{equation*}
    \theta \T \Phi  = \Pi_m T^{(\lambda)} (\theta \T \Phi).
\end{equation*}
In this paper, we establish that the projected Bellman operator $\Pi_m T^{(\lambda)}$ is a contraction, which allows us to bound the error $\| \Phi \T \theta^* - V^\pi \|_m$.
%We follow the notation of \citep{SuttonMW15}:
%\begin{itemize}
%	\item
%	\item $\phi(s)$ is the feature vector of state $s \in S$. $\Phi$ is the matrix whose columns are the feature vectors for all states.
%	\item A Euclidean weighted $d$-norm $\| \cdot \|_d$ is defined by:
%	\begin{equation}
%	\| v \|^2_d = \sum_i d(i) v^2(i)
%	\end{equation}
%	\item $\Pi_d$ is the projection operator into the subspace spanned by $\phi(s)$ with respect to the $d$-norm.
%	\item $V^\pi(s)$ is the value function of policy $\pi$ starting at state $s$:
%	\begin{equation}
%	v(s) = \mathbb{E}_\pi \left[ \sum_{t=0}^\infty \gamma^t r_t \given[\Big]  s_0 = s \right]
%	\end{equation}
%	\item $\mu$ is the behavior policy which generated the samples, $\pi$ is the target policy which we want to evaluate. $P_\mu, P_\pi$ are the corresponding induced transition matrices, and $d_\mu, d_\pi$ are the corresponding induced stationary distributions over states.
%	\item $T(v) = r + \gamma P_\pi v$ is the Bellman operator. $\Pi T (v) = \Pi (r + \gamma P v)$ is the projected Bellman operator.
%	\item $f = (I - \gamma P_\pi)^{-1} d_\mu$ is the empathic weight vector, and $F = diag(f)$.
%\end{itemize}

\section{Results}
We start from ETD($0$). It is well known that $T^\pi$ is a $\gamma$-contraction with respect to the $d_\pi$-weighted Euclidean norm \citep{BT96}. However, it is not immediate that the concatenation $\Pi_f T^\pi$ is a contraction in any norm. Indeed, for the TD(0) algorithm \cite{sutton_reinforcement_1998}, a similar representation as a projected Bellman operator holds, but it may be shown that in the off-policy setting the algorithm diverges \citep{baird1995residual}.
%, and that the projection operator weighted Euclidean norm is non-expanding in that norm \citep{BT96}.
%In the on-policy setting, in which $\mu=\pi$, the projection is with respect to the $d_\pi$-weighted Euclidean norm and the projected Bellman operator $\Pi_{d_\pi}T$ is a $\gamma$-contraction. (THIS IS FOR TD(0))

%However,
%in the off-policy case, the projection is with respect to a different norm, and the norm mismatch prevents concatenation in order to obtain a contracting projected Bellman operator. We show here that for the empathic weighted norm, $T$ is still a contraction.
%
%We define the empathic projected Bellman operator by $\Pi_f T(v)$.
The following theorem shows that for ETD($0$), the projected Bellman operator $\Pi_f T^\pi$ is indeed a contraction.
\begin{Theorem}\label{thm:one}
Denote by $\kappa = \min_s \frac{d_\mu (s)}{f(s)} $, then $\Pi_f T^\pi$ is a $\sqrt{\gamma (1 - \kappa)}$-contraction with respect to the Euclidean $f$-weighted norm, namely,
\begin{equation*}
    \| \Pi_f T^\pi v_1 - \Pi_f T^\pi v_2 \|_f \leq \sqrt{\gamma(1-\kappa)} \|v_1 - v_2\|_f,\quad \forall v_1,v_2\in \R^{|S|}.
\end{equation*}
\end{Theorem}

\begin{proof}
Let $F = diag(f)$. We have
\begin{equation}
\begin{split}
\| v \|^2_f - \gamma \| P_\pi v \|^2_f &= v^\top F v - \gamma v^\top P_\pi^\top F P_\pi v  \\
&\geq^{a} v^\top F v - \gamma v^\top diag(f^\top P_\pi)v \\
& = v^\top [F - \gamma diag(f^\top P_\pi)]v \\
& = v^\top \left[diag \left(f^\top (I-\gamma P_\pi) \right)  \right] v \\
&=^{b} v^\top diag(d_\mu) v = \| v \|^2_{d_\mu},
\end{split}
\end{equation}
where (a) follows from the Jensen inequality:
\begin{equation}
\begin{split}
v^\top P_\pi^\top F P_\pi v &= \sum_s f(s) ( \sum_{s'} P_\pi(s'|s) v(s') )^2 \\
&\leq \sum_s f(s)  \sum_{s'}P_\pi(s'|s)  v^2(s') \\
&= \sum_{s'} v^2(s') \sum_s f(s) P_\pi(s'|s) \\
&= v^\top diag(f^\top P_\pi)v,
\end{split}
\end{equation}

and (b) is by the definition of $f$ in \eqref{eq:f}.

Notice that for every $v$:
\begin{equation}
\| v \|^2_{d_\mu} = \sum_s d_\mu(s) v^2(s)  \geq \sum_s \kappa f(s) v^2(s)   = \kappa \| v \|^2_f
\end{equation}

Therefore:
\begin{equation}
\begin{split}
\| v \|^2_f & \geq \gamma \|P_\pi v \| ^2_f + \| v \|^2_{d_\mu} \geq \gamma \|P_\pi v \| ^2_f + \kappa \| v \|^2_f, \\
\Rightarrow & \quad \gamma \|P_\pi v\|^2_f \leq (1-\kappa) \| v \|^2_f
\end{split}
\end{equation}
and:
\begin{equation}
\begin{split}
  \|T^\pi v_1 - T^\pi v_2 \|^2_f  &= \| \gamma P_\pi(v_1 - v_2) \| ^2_f \\
  & = \gamma^2 \| P_\pi (v_1 - v_2) \|^2_f \\
  &\leq \gamma (1-\kappa) \| v_1 - v_2 \| ^2_f .
\end{split}
\end{equation}
Hence, $T$ is a $\sqrt{\gamma(1-\kappa)}$-contraction. Since $\Pi_f$ is a non-expansion in the $f$-weighted norm \citep{BT96}, $\Pi_f T$ is a $\sqrt{\gamma(1-\kappa)}$-contraction as well.
\end{proof}
Notice that $\kappa$ obtains values ranging from $\kappa = 0$ (when there is a state visited by the target policy, but not the behavior policy), to $\kappa = 1-\gamma$ (when the two policies are identical). In the latter case we obtain the classical bound: $\sqrt{\gamma(1-\kappa)}=\gamma$. This result resembles that of \cite{kolter2011fixed} who used the discrepancy between the behavior and the target policy to bound the TD-error.

An immediate consequence of Theorem \ref{thm:one} is the following error bound, based on Lemma 6.9 of \citet{BT96}.
\begin{corollary}\label{corr:err}
We have
\begin{equation*}
    \| \Phi \T \theta^* - V^\pi \|_f \leq \frac{1}{1 - \sqrt{\gamma (1-\kappa)}} \| \Pi_f V^\pi - V^\pi \|_f .
\end{equation*}
\end{corollary}
In a sense, the error $\| \Pi_f V^\pi - V^\pi \|_f$ is the best approximation we can hope for, within the capability of our linear approximation architecture. Corollary \ref{corr:err} guarantees that we are not too far away from it.

Now we move on to the analysis of ETD($\lambda$):

\begin{Theorem}\label{thm:two}
	$\Pi_m T^{(\lambda)}$ is a $\sqrt{\beta}$-contraction with respect to the Euclidean $f$-weighted norm, where $\beta = \frac{\gamma (1-\lambda)}{1-\lambda \gamma}$. Namely,
	\begin{equation*}
	\| \Pi_m T^{(\lambda)} v_1 -  \Pi_m T^{(\lambda)} v_2 \|_m \leq \sqrt{\beta} \|v_1 - v_2\|_m,\quad \forall v_1,v_2\in \R^{|S|}.
	\end{equation*}
\end{Theorem}

\begin{proof}
	The proof is almost identical to the proof of Theorem $\ref{thm:one}$, only now we cannot apply Jensen's inequality directly, since the rows of $P^\lambda_\pi$ do not sum to $1$. However:
	\begin{equation}
	P^\lambda_\pi \textbf{1} = \left( I - (I-\gamma \lambda P_\pi)^{-1} (I-\gamma P_\pi) \right) \textbf{1} = \beta \textbf{1},
	\end{equation}
	and each entry of $P^\lambda_\pi$ is positive. Therefore $\frac{P^\lambda_\pi}{\beta}$ will hold for Jensen's inequality.
	Let $M = diag(m)$, we have
	\begin{equation}
	\begin{split}
	\| v \|^2_m - \frac{1}{\beta} \| P_\pi v \|^2_m &= v^\top M v - \beta v^\top \frac{P^\lambda_\pi}{\beta}^\top M \frac{P^\lambda_\pi}{\beta} v  \\
	&\geq^{a} v^\top M v - \beta v^\top diag(m^\top \frac{P^\lambda_\pi}{\beta})v \\
	& = v^\top [M -  diag(m^\top P^\lambda_\pi)]v \\
	& = v^\top \left[diag \left(m^\top (I- P^\lambda_\pi) \right)  \right] v \\
	&=^{b} v^\top diag(\textbf{i}) v = \| v \|^2_{\textbf{i}},
	\end{split}
	\end{equation}
	where (a) follows from the Jensen inequality and (b) from Equation \ref{eq:m}.
	
	Therefore:
	\begin{equation}
	\| v \|^2_m \geq \frac{1}{\beta} \| P^\lambda_\pi v \| ^2_m + \| v \|^2_{\textbf{i}} \geq \frac{1}{\beta} \| P^\lambda_\pi v \| ^2_m,
	\end{equation}
	and:
	\begin{equation}
	\|T^{(\lambda)} v_1 - T^{(\lambda)} v_2 \|^2_m  = \| P^\lambda_\pi (v_1 - v_2) \| ^2_m \leq \beta \| v_1 - v_2 \| ^2_m .
	\end{equation}
	Hence, $T^{(\lambda)}$ is a $\sqrt{\beta}$-contraction. Since $\Pi_m$ is a non-expansion in the $m$-weighted norm \citep{BT96}, $\Pi_m T^{(\lambda)}$ is a $\sqrt{\beta}$-contraction as well.
\end{proof}

As before, Theorem \ref{thm:two} leads to the following error bound, based on Theorem 1 of \citet{tsitsiklis1997analysis}.
\begin{corollary}\label{corr:err2}
	We have
	\begin{equation*}
	\| \Phi \T \theta^* - V^\pi \|_m \leq \frac{1}{1 - \sqrt{\beta}} \| \Pi_m V^\pi - V^\pi \|_m .
	\end{equation*}
\end{corollary}

We now show in an example that our contraction modulus bounds are tight.
\paragraph{Example} Consider an MDP with two states: Left and Right. In each state there are two identical actions leading to either Left or Right deterministically. The behavior policy will choose Right with probability $\epsilon$, and the target policy will choose Left with probability $\epsilon$. Calculating the quantities of interest:
\begin{equation*}
\begin{split}
P_\pi = \left( \begin{array}{cc}
\epsilon & 1-\epsilon  \\
\epsilon & 1-\epsilon
\end{array} \right)
, \quad d_\mu = \left( 1-\epsilon, \epsilon \right) \\
f = \frac{1}{1-\gamma} \left( 1 + 2 \epsilon \gamma - \epsilon - \gamma , -2 \epsilon \gamma + \epsilon + \gamma \right) \T.
\end{split}
\end{equation*}
So for $v = \left( 0, 1 \right) \T$:
\begin{equation*}
\| v \|^2_f = \frac{\epsilon + \gamma  - 2\epsilon\gamma}{1-\gamma}, \quad \| P_\pi v \|^2_f = \frac{ (1-\epsilon)^2 }{1-\gamma},
\end{equation*}
and for small $\epsilon$ we obtain that $\frac{\| \gamma P_\pi v \|^2}{\| v \|^2_f} \approx \gamma $.

\section{Discussion}

Interestingly, the ETD error bounds in Corollary \ref{corr:err} and \ref{corr:err2} are more conservative by a factor of square root than the error bounds for standard on-policy TD \citep{BT96, tsitsiklis1997analysis}. Thus, it appears that there is a price to pay for off-policy convergence. Future work should address the implications of the different norms in these bounds.

Nevertheless, we believe that the results in this paper motivate ETD (or its least-squares counterpart; \citealt{yu2015etd}) as the method of choice for off-policy policy-evaluation in MDPs.

\bibliography{ContractionBib}
\bibliographystyle{icml2015}

\end{document}